\documentclass[final]{alt2026} 

\usepackage{url}            
\usepackage{booktabs}       
\usepackage{multirow}
\usepackage{array}  
\usepackage{amsfonts}       
\usepackage{nicefrac}       
\usepackage{microtype}      
\usepackage{xcolor}         
\usepackage{soul}
\usepackage{paralist}

\usepackage{cleveref}

\usepackage{tikz}
\usetikzlibrary{shapes,arrows,positioning,intersections}
\usepackage{physics}
\usepackage[bb=dsserif]{mathalpha} 
\usepackage{algorithm}
\usepackage[noend]{algorithmic}

\usepackage{graphicx}
\usepackage{subcaption} 

\usepackage{mathtools}
\usepackage{enumitem}

\usepackage{amsmath,amsfonts,bm}

\def\1{\bm{1}}

\DeclareMathAlphabet{\mathsfit}{\encodingdefault}{\sfdefault}{m}{sl}
\SetMathAlphabet{\mathsfit}{bold}{\encodingdefault}{\sfdefault}{bx}{n}

\def\gA{{\mathcal{A}}}

\def\gX{{\mathcal{X}}}

\def\sW{{\mathbb{W}}}

\newcommand{\E}{\mathbb{E}}

\newcommand{\R}{\mathbb{R}}

\newcommand{\defeq}{\vcentcolon=} 

\DeclareMathOperator*{\argmin}{arg\,min}

\DeclarePairedDelimiterX{\infdivx}[2]{(}{)}{%
  #1\;\delimsize\|\;#2%
}

\DeclarePairedDelimiterX{\inp}[2]{\langle}{\rangle}{#1, #2}

\newcounter{protocol}

\newenvironment{nalign}{
    \begin{equation}
    \begin{aligned}
}{
    \end{aligned}
    \end{equation}
    \ignorespacesafterend
}

\usepackage[textsize=tiny]{todonotes}

\title[Understanding Adam's Momentum Factors]{How to Set $\beta_1, \beta_2$ in Adam: An Online Learning Perspective}
\usepackage{times}

\altauthor{%
 \Name{Quan Nguyen} \Email{manhquan233@gmail.com}\\
 \addr University of Victoria, Canada
}

\begin{document}

\maketitle

\begin{abstract}%
  While Adam is one of the most effective optimizer for training large-scale machine learning models, a theoretical understanding of how to optimally set its momentum factors, $\beta_1$ and $\beta_2$, remains largely incomplete. 
  Prior works have shown that Adam can be seen as an instance of Follow-the-Regularized-Leader (FTRL), one of the most important class of algorithms in online learning. 
  The prior analyses in these works required setting $\beta_1 = \sqrt{\beta_2}$, which does not cover the more practical cases with $\beta_1 \neq \sqrt{\beta_2}$.
  We derive novel, more general analyses that hold for both $\beta_1 \geq \sqrt{\beta_2}$ and $\beta_1 \leq \sqrt{\beta_2}$. 
  In both cases, our results strictly generalize the existing bounds. 
  Furthermore, we show that our bounds are tight in the worst case.
  We also prove that setting $\beta_1 = \sqrt{\beta_2}$ is optimal for an oblivious adversary, but sub-optimal for an non-oblivious adversary.
\end{abstract}


\section{Introduction}
Training a neural network is an instance of a nonsmooth nonconvex optimization problem, where the goal is to find a $w^*$ that minimizes a function $F: \sW \to \R$, where $\sW$ is the set of possible solutions.
A training algorithm usually starts from some initial solution $w_0$, and then iteratively update the solution $w_{t+1} = w_t + \Delta_t$.
Here, $\Delta_t$ specifies the update in round $t$.
The most efficient training algorithms are first-order methods, which computes $\Delta_t$ based on the (possibly stochastic) gradients $g_t$, where $\E[g_t] = \frac{dF}{dw_t}$.
The Adam optimizer~\citep{KingmaAndBa2017adammethodstochasticoptimization} is among the most popular first-order methods, which computes a coordinate-wise update of the form
\begin{align}
  \Delta_t = -\alpha_t \frac{\sum_{s = 0}^{t-1}\beta_1^{t-1-s} g_s}{\sqrt{\sum_{s= 0}^{t-1} \beta_2^{t-1-s} g_s^2}},
  \label{eq:AdamUpdate}
\end{align}
where $\alpha_t > 0$ is the learning rate of the optimization problem, and $0 < \beta_1, \beta_2 < 1$ are the first and second-order momentum discount factors.
As justified by~\cite{AhnNeurIPS2024AdamisEffectiveNonConvex}, we omit the bias-correction terms since they are coordinate-independent and can be absorbed into $\alpha_t$.

Ever since Adam was introduced, numerous experimental results~\citep[see e.g.][]{Orvieto2025SearchAdamSecretSauce} have shown that tuning $\beta_1$ and $\beta_2$ play a vital role in the empirical performance of Adam.
Despite its practical importance, a theoretical foundation for tuning these discount factors is still lacking, and current practitioners mostly rely on either expensive grid search to tune these factors.
Recent empirical findings~\citep{Orvieto2025SearchAdamSecretSauce} found that for each $\beta_1$, the optimal value of $\beta_2$ often satisfies $\beta_1 \leq \sqrt{\beta_2}$, and that $\beta_1 = \sqrt{\beta_2}$ is in fact optimal in a number of experiments. Figure~\ref{fig:betascorr} in the appendix illustrates an example.

A recently emerged framework for nonconvex optimization is the online-to-nonconvex framework~\citep{CutkoskyICML2023Online2Nonconvex}, which casts a new perspective on $\Delta_t$. 
In this framework, $\Delta_t$ is the prediction in round $t$ of an online learning algorithm in a $1$-dimensional online linear optimization problem.
In this framework,~\citet{AhnAdamisFTRL2024} recently showed that the Adam's update rule in~\Cref{eq:AdamUpdate} corresponding to the output of an Follow-the-Reguralized-Leader (FTRL) algorithm~\citep[see e.g.][]{OrabonaIntroToOnlineLearningBook} with the sequence of losses $\ell_t(x) = \beta_1^{-t}g_t x$ and the square function $\frac{x^2}{2}$ as the regularizer.
More concretely,~\Cref{eq:AdamUpdate} is equivalent to
\begin{align}
  \Delta_t = \argmin_{x \in \R} \frac{1}{2\eta_t} x^2 + \sum_{s=0}^{t-1} \beta_1^{-s}g_s x = -\eta_t \sum_{s=0}^{t-1} \beta_1^{-s} g_s,
\end{align}
where $\eta_t = \alpha_t \frac{(\beta_1/\sqrt{\beta_2})^{t-1}}{\sqrt{\sum_{s=0}^{t-1} \beta_2^{-s}g_s^2}}$ is the learning rate of the online learning problem.
The full procedure is given in Algorithm~\ref{algo:AdamfromFTRL}.

~\citet{AhnAdamisFTRL2024} considered a special version of Adam with $\beta_1 = \sqrt{\beta_2}$, and proved an upper bound on its $\beta_1$-discounted regret $R_{T, \beta_1}$. Formally, let $D > 0$ and $\gX = [-D, D]$ be a set that contains all the possible values of $\Delta_t$. Let $u \in \gX$. 
The $\beta_1$-discounted regret with respect to $u$ after $T$ rounds is
\begin{align}
    R_{T, \beta_1}(u) = \beta_1^{T}\sum_{t=1}^T \beta_1^{-t} g_t(\Delta_t - u).
    \label{eq:discountedRegretFormula}
\end{align}
By setting $\beta_2 = \beta_1^2$,~\citet{AhnAdamisFTRL2024} and~\citet{AhnNeurIPS2024AdamisEffectiveNonConvex} showed that Adam obtains $R_{T, \beta_1}(u) \leq O(D\sqrt{\sum_{t=0}^T \beta_1^{T-t}g_t^2})$ for all $u \in \gX$, which then led to optimal convergence rate on nonsmooth nonconvex optimization problems. 
While these results are significant, they only hold for the case $\beta_1 = \sqrt{\beta_2}$ and do not hold for more general settings.
In fact, both of these works state that the analysis for the general cases that include $\beta_1 \neq \sqrt{\beta_2}$ is an important open problem.
Furthermore, existing results do not explain why $\beta_1 = \sqrt{\beta_2}$ is a good choice, and how optimal or sub-optimal this choice is.

\textbf{Contributions.} This work takes a step towards addressing the aforementioned open problem posed by~\citet{AhnAdamisFTRL2024} and~\citet{AhnNeurIPS2024AdamisEffectiveNonConvex}.  
More specifically, we focus on deriving regret upper bounds on $R_{T, \beta_1}(u)$ in~\Cref{eq:discountedRegretFormula} for Adam with $\beta_1 \neq \sqrt{\beta_2}$.
Our main contributions are as follows.
\begin{itemize}
  \item In Section~\ref{sec:beta1smallerthansqrtbeta2}, under the condition $\beta_1 \leq \sqrt{\beta_2}$, Corollary~\ref{corolary:constantAlpha} shows a new $\beta_1$-discounted regret bound of order 
  $O(D \frac{\sqrt{\beta_2}}{\beta_1} \sqrt{\sum_{t=0}^T \beta_2^{T-t}g_t^2} + D\max_{t}\abs{\beta_1^{T-t}g_t})$. 
  This bound strictly generalizes existing bounds in~\citet{AhnAdamisFTRL2024} and~\citet{AhnNeurIPS2024AdamisEffectiveNonConvex}.
  Our proof of this result is not based on clipping as in~\citet{AhnAdamisFTRL2024}, but relies on a generalization of the analysis for the Scale-Free FTRL algorithm~\citep{Orabona2015scalefreealgorithmsonlinelinear}.
  Furthermore, Theorem~\ref{theorem:tight} proves that the new generalized bound is tight.
  \item In Section~\ref{sec:beta1largerthansqrtbeta2}, we study Adam with $\beta_1 \geq \sqrt{\beta_2}$. 
  Theorem~\ref{theorem:RegretBoundPlargerthan1} shows that an exponentially decaying sequence of $\alpha_t$ leads to an upper bound of order $O(D\sqrt{\sum_{t=0}^T \beta_1^{2T}\beta_2^{-t}g_t^2} + D\max_{t}\abs{\beta_1^{T-t}g_t})$ for the $\beta_1$-discounted regret.
  Both this bound and the analysis also recover the existing bounds and analyses in~\citet{AhnAdamisFTRL2024} and~\citet{AhnNeurIPS2024AdamisEffectiveNonConvex} when $\beta_1 = \sqrt{\beta_2}$.
  \item While the results in~\Cref{sec:beta1smallerthansqrtbeta2,sec:beta1largerthansqrtbeta2} suggest that $\beta_1 = \sqrt{\beta_2}$ is optimal under an oblivious adversary, Theorem~\ref{theorem:notoptimal} in Section~\ref{sec:notoptimalNonObliviousAdversary} shows an instance where $\beta_1 = \sqrt{\beta_2}$ is provably sub-optimal compared to $\beta_1 < \sqrt{\beta_2}$. Our construction uses an non-oblivious adversary, indicating that optimally tuning $\beta_1, \beta_2$ strongly depends on the adversarial nature of the underlying environment, and that dynamically tuning the momentum factors might be beneficial.
\end{itemize}
\textbf{Notations.} For an integer $T$, we denote $[T] = \{1, 2, \dots, T\}$. 
For two quantities $f$ and $g$, we write $f \lesssim g, f \gtrsim g$ and $f \simeq g$ to denote $f = O(g), g = O(f)$ and $f = \Theta(g)$, respectively.

Throughout the paper, we write $p = \frac{\beta_1}{\sqrt{\beta_2}}, v_s = \beta_1^{-s}g_{s}$ and $D_T = \max_{t \in [T]}\abs{\Delta_t}$.
Note that $D_T \leq D$.
The regularized objective function in round $t$ of the FTRL algorithm is 
\begin{align}
    F_t(x) = \frac{1}{2\eta_t}x^2 + \sum_{s=0}^{t-1} \beta_1^{-s}g_s x = \frac{1}{2\eta_t}x^2 + \sum_{s=0}^{t-1} v_s x.
\end{align}
Note that $\Delta_t = \argmin_{x \in \gX} F_t(x) = {\rm clip}_D(-\eta_t \sum_{s=0}^{t-1}v_s)$.
Here, the clipping operation is ${\rm clip}(x) = x \min(\frac{D}{\abs{x}}, 1)$. 
Also, $F_{t+1}(x) = F_t(x) - \frac{1}{2\eta_t}x^2 + v_{t}x + \frac{1}{2\eta_{t+1}}x^2$.
The standard FTRL regret after $T$ rounds is
\begin{align}
  R_{T}(u) \defeq \sum_{t=1}^T v_t (\Delta_t - u) = \sum_{t=1}^T \beta_1^{-t} g_t(\Delta_t - u).
\end{align}
By definition, an upper bound on $R_T(u)$ implies an upper bound on the $\beta_1$-discounted regret $R_{T, \beta_1}(u)$, since $R_{T, \beta_1}(u) = \beta_1^T R_T(u)$. 
We can observe that $R_T(u)$ is the standard non-discounted regret of an online learning problem that predicts $\Delta_t$ and observes the loss $v_t$ in round $t$.
This learning procedure with non-discounted losses is in Algorithm~\ref{algo:AdamfromFTRLWrittenInV}, which is essentially Algorithm~\ref{algo:AdamfromFTRL} re-written in standard online learning notations.


\subsection{Related Works}
Due to the vast amount of related literature on Adam and Adam-like methods, we focus on the most relevant works on the online-to-nonconvex framework and on tuning $\beta_1$, $\beta_2$.
A more comprehensive review of Adam on a diverse range of topics such as implicit bias, separation from gradient descent, the SignDescent hypothesis and empirical performance, can be found in~\citet{Orvieto2025SearchAdamSecretSauce, JinICML2025AComprehensiveFrameworkAnalyzingConvergenceAdam, Vasudeva2025richsimpleimplicitbias} and references therein.

\textbf{Online-To-Nonconvex Optimization.} Viewing $\Delta_t$ as the output of an online learning algorithm was first considered by~\citet{CutkoskyICML2023Online2Nonconvex}, who showed that a small shifting regret of an online learner can be converted to an optimal convergence rate for finding a stationary points.
Instead of shifting regret, subsequent works~\citep{ZhangCutkoskyICML2024RandomScalingAndMomentum, AhnICML2025ScheduleFreeSGDEffectiveOnlineToNonConvex, AhnNeurIPS2024AdamisEffectiveNonConvex} considered $\beta$-discounted regret and showed that this notion of regret can also be used to bound the convergence rate of an optimization algorithm.
By using online mirror descent as the online learner,~\citet{ZhangCutkoskyICML2024RandomScalingAndMomentum, AhnICML2025ScheduleFreeSGDEffectiveOnlineToNonConvex} obtained variants of gradient descent with momentum with optimal guarantee for nonconvex nonsmooth optimization.
\citet{AhnAdamisFTRL2024, AhnNeurIPS2024AdamisEffectiveNonConvex} used FTRL as the online learner to derive Adam with $\beta_1 = \sqrt{\beta_2}$, and then proved its optimal convergence rates for nonconvex optimization.
These works also provided theoretical justifications for why $\beta_1$ and $\beta_2$ should be set close, but not exactly equal, to $1$.

\textbf{Setting momentum factors $\beta_1$ and $\beta_2$.}
Adam was original proposed by~\citet{KingmaAndBa2017adammethodstochasticoptimization}, who recommended several settings of the momentum factors such as $1-\beta_1 = \sqrt{1-\beta_2}$, $\beta_1 < \sqrt[4]{\beta_2}$ and $\beta_1 = 0.9, \beta_2 = 0.999$.
The convergence issues of Adam and its variants with $\beta_1 < \sqrt{\beta_2}$ are studied in~\citet{Reddi2018AdamNotConverge, Zhang2022AdamCanConverge, AlacaogluICML2020constantBeta1works}.
Recently,~\citet{Taniguchi2024adoptAnyBeta2} showed that a modified version of Adam can converge for any $\beta_2$.
\cite{Orvieto2025SearchAdamSecretSauce} conduced some of the most comprehensive empirical study on the effects of tuning $\beta_1, \beta_2$ and suggested that optimal tuning would require a strong correlation between $\beta_1$ and $\beta_2$.

\begin{algorithm}[t]
	\KwIn{$\beta_1, \beta_2 \in (0,1), (\alpha_t)_t > 0$, decision space $\gX \subseteq \R$}
  Receive $g_0 \neq 0$\\
	\For{$t = 1, \dots, T$}{        \;
    Compute $\eta_t = \alpha_t \frac{(\beta_1/\sqrt{\beta_2})^{t-1}}{\sqrt{\sum_{s=0}^{t-1} \beta_2^{-s}g_s^2}}$ \\
    Compute $\Delta_t = \argmin_{x \in \gX} \frac{1}{2\eta_t} x^2 + \sum_{s=0}^{t-1} \beta_1^{-s}g_s x$ \\     
    Receive $g_t \in \R$ and incur loss $\ell_t = \beta_1^{-t}g_t \Delta_t$\;
	}
	\caption{Adam as Follow-the-Regularized-Leader on discounted losses}
	\label{algo:AdamfromFTRL}
\end{algorithm}
\begin{algorithm}[t]
	\KwIn{$p \in (0,1), (\alpha_t)_t > 0$ where $\alpha_{t+1} \leq \alpha_t$, decision space $\gX \subseteq \R$}
  Receive $v_0 \neq 0$\\
	\For{$t = 1, \dots, T$}{        \;
    Compute $\eta_t = \alpha_t \frac{p^{t-1}}{\sqrt{\sum_{s=0}^{t-1} (p^sv_s)^2}}$ \\
    Compute $\Delta_t = \argmin_{x \in \gX} \frac{1}{2\eta_t} x^2 + \sum_{s=0}^{t-1} v_s x$ \\     
    Receive $v_t \in \R$ and incur loss $\ell_t = v_t \Delta_t$\;
	}
	\caption{Follow-the-Regularized-Leader on sequence of losses $(v_t)_t$}
	\label{algo:AdamfromFTRLWrittenInV}
\end{algorithm}

\section{Discounted Regret Bound of Adam with $\beta_1 \leq \sqrt{\beta_2}$}
\label{sec:beta1smallerthansqrtbeta2}
In this section, we derive a bound on the $\beta_1$-discounted regret $R_{T, \beta_1}(u)$ of Adam with for $\beta_1 \leq \sqrt{\beta_2}$ and constant $\alpha$.
First, in~\Cref{sec:generalAnalysisFixedAlpha}, we show an general analysis that hold for both bounded $\gX = [-D, D]$ and unbounded $\gX = \R$ domain, as well as both oblivious and non-oblivious adversary. 
Then, in~\Cref{sec:howtight}, we show that our new analysis and regret bound cannot be improved further.

\subsection{A More General Regret Analysis for both Bounded and Unbounded Domain}
\label{sec:generalAnalysisFixedAlpha}
Recall that $p = \frac{\beta_1}{\sqrt{\beta_2}}$ and $D_T = \max_{t \in [T]}\abs{\Delta_t}$.
We have $p \leq 1$ throughout this section.
The following theorem states the regret bound of~\Cref{algo:AdamfromFTRL} with $\beta_1 \leq \sqrt{\beta_2}$.
\begin{theorem}
  For any $T \geq 2, \beta_1 \leq \sqrt{\beta_2}$, any sequence $(\alpha_t)_t$ where $\alpha_{t+1} \leq \alpha_t$ and any sequence of $(g_t)_{t=0,\dots,T}$,~\Cref{algo:AdamfromFTRL} guarantees
  \begin{align}
    R_T(u) \leq \frac{u^2 }{\alpha_{T+1}\beta_1^T}\sqrt{\sum_{t=0}^T \beta_2^{T-t}g_t^2} + \frac{\sqrt{6\beta_2}}{2\beta_1}\left(\max_{t \in [T]}\frac{\alpha_t\sqrt{\beta_2^t}}{\beta_1^t}\right)\sqrt{\sum_{t=0}^T \beta_2^{-t}g_t^2} + 7D_T\max_{0 \leq t \leq T}\abs{\beta_1^{-t}g_t}.    
    \label{eq:generalbound}
  \end{align}
  \label{theorem:GeneralRegretBoundPlessthan1}
\end{theorem}
\looseness=-1 Before showing the proof of~\Cref{theorem:GeneralRegretBoundPlessthan1}, we discuss the generality of the bound in~\Cref{eq:generalbound} in comparison to existing bound in~\citet{AhnAdamisFTRL2024,AhnNeurIPS2024AdamisEffectiveNonConvex}. 
Letting $(\alpha_t)_t = \alpha$ be constant and noting that $\max_{t \in [T]}\frac{\sqrt{\beta_2^t}}{\beta_1^t} = \max_{t \in [T]} \frac{1}{p^t} = \frac{1}{p^T}$ for $p \leq 1$, we obtain the following result.
\begin{corollary}
  Let $\alpha > 0$ be a constant. For any $T \geq 2, \beta_1 \leq \sqrt{\beta_2}$ and any sequence of $(g_t)_{t=0,\dots,T}$,~\Cref{algo:AdamfromFTRL} with $\alpha_t = \alpha$ guarantees
  \begin{align}
    R_T(u) \leq \left(\frac{u^2 }{\alpha} + \frac{\alpha\sqrt{6\beta_2}}{2\beta_1}\right)\frac{1}{\beta_1^T}\sqrt{\sum_{t=0}^T \beta_2^{T-t}g_t^2} + 7D_T\max_{0 \leq t \leq T}\abs{\beta_1^{-t}g_t}.
    \label{eq:boundWithConstantAlpha}
  \end{align}
  \label{corolary:constantAlpha}
\end{corollary}
\begin{remark}
      Multiplying $\beta_1^T$ on both sides of~\Cref{eq:boundWithConstantAlpha} results in the following bound for the $\beta_1$-discounted regret:
    \begin{align}
        R_{T, \beta_1}(u) = \beta_1^T R_T(u) \leq \left(\frac{u^2}{\alpha} + \frac{\alpha \sqrt{6\beta_2}}{2\beta_1}\right) \sqrt{\sum_{t=0}^T \beta_2^{T-t}g_t^2} + 7D_T \max_{0 \leq t \leq T} \abs{\beta_1^{T-t}g_t}.
    \end{align}
    Ignoring the (small) constant factors, setting $\beta_2 = \beta_1^2$ leads to the bound in~\citet[][Theorem B.2]{AhnAdamisFTRL2024}.
    In addition, on the bounded domain $\gX = [-D, D]$, setting $\beta_2 = \beta_1^2$, $\abs{u} = D$ and $\alpha = \Theta(D)$ recovers the $O(D\sqrt{\sum_{t=1}^T \beta_1^{T-t}g_t^2})$ bound in~\citet[][Theorem 9]{AhnNeurIPS2024AdamisEffectiveNonConvex}.
    \label{remark3}
\end{remark}
\begin{remark}
  Under an oblivious adversary, i.e. when the sequence $(g_t)_t$ is fixed regardless of the choice of $\beta_2$, choosing $\beta_2 = \beta_1^2$ is optimal for the regret bound in~\Cref{eq:boundWithConstantAlpha}.
  This follows from the fact that $T - t \geq 0$ and thus, $\beta_2^{T-t} \geq \beta_1^{2(T-t)}$ for all $t \in [T]$.
\end{remark}

\subsection{Proof of Theorem~\ref{theorem:GeneralRegretBoundPlessthan1}}
As mentioned above, ~\Cref{algo:AdamfromFTRLWrittenInV} is exactly~\Cref{algo:AdamfromFTRL} written in standard online learning notations.
Hence, we will focus on analyzing the regret $R_T(u) = \sum_{t=1}^T v_t(\Delta_t - u)$ of the FTRL with time-varying learning rate in~\Cref{algo:AdamfromFTRLWrittenInV}.
Our proof is essentially a more general version of the proof for the Scale-Free FTRL algorithm in~\citet{Orabona2015scalefreealgorithmsonlinelinear}.
We first give a proof sketch that highlights the main steps of the analysis, and then give the full proof.

\noindent
\textbf{Proof Sketch.}
Our first step is ensure that the sequence of learning rates $(\eta_t)_t$ is non-increasing, which indeed follows from the fact that $p \leq 1$ and $\eta_{t+1} \leq \eta_t$.
Having a non-increasing sequence of learning rates, we can employ the standard FTRL analysis~\citep{OrabonaIntroToOnlineLearningBook} and obtain
\begin{align*}
  R_T(u) \leq \frac{u^2 \sqrt{\sum_{t=1}^T \beta_2^{-t}g_t^2}}{\alpha_{T+1} p^T} + \sum_{t=1}^T \min\left\{\underbrace{\frac{\alpha_t}{2} \frac{p^{t-1} v_t^2}{\sqrt{\sum_{s=0}^{t-1} (p^s v_s)^2}}}_{(a)}, \underbrace{2D_T \abs{v_t}}_{(b)}\right\}.
\end{align*}
Observe that if $p = 1$, then $(a)$ would recover the well-known quantity $\sum_{t=1}^T \frac{v_t^2}{\sqrt{\sum_{s=0}^{t-1}v_s^2}}$ that frequently appears in the analysis of FTRL with adaptive learning rates~\citep[e.g.][]{ItoCOLT2024} and Scale-Free FTRL in particular~\citep{Orabona2015scalefreealgorithmsonlinelinear,AhnAdamisFTRL2024}.
Existing works proceed by considering two different cases on the magnitude of $\abs{v_t}$, thereby creating two different telescopic sums.
If $\abs{v_t}$ is large, then the first term $(a)$ can be very large and we resort to bounding $(b) \lesssim O(D_T \max_{s \leq t} \abs{v_s} - \max_{s \leq t-1} \abs{v_s} )$.
On the other hand, if $\abs{v_t}$ is small then we can bound $(a) \leq O(\sqrt{\sum_{s \leq t} v_s^2} - \sqrt{\sum_{s \leq t-1}v_s^2})$.

In our more general case of $p \leq 1$, we need a more fine-grained analysis.
Instead of $\abs{v_t}$, we examine $p^{t-1}\abs{v_t}$. 
In the first case where $p^{t-1}\abs{v_t}$ is large, it turns out that we still have $(b) \lesssim O(D_T \max_{s \leq t} \abs{v_s} - \max_{s \leq t-1} \abs{v_s} )$ for $p \leq 1$.
On the contrary, when $p^{t-1}\abs{v_t}$ is small then $(p^tv_t)^2$ is small, which allows us to bound $(a) \leq O(\sqrt{\sum_{s \leq t} (p^sv_s)^2} - \sqrt{\sum_{s \leq t-1}(p^sv_s)^2})$.

\begin{proofof}{\Cref{theorem:GeneralRegretBoundPlessthan1}}
The standard FTRL analysis~\citep[e.g.][Lemma 7.1]{OrabonaIntroToOnlineLearningBook} shows that
\begin{align}
    R_T(u) &= \frac{u^2}{\eta_{T+1}} + F_{T+1}(\Delta_{T+1}) - F_{T+1}(u) + \sum_{t=1}^{T} F_t(\Delta_t) - F_{t+1}(\Delta_{t+1}) + v_{t}\Delta_t \\
    &\leq \frac{u^2}{\eta_{T+1}} + \sum_{t=1}^{T} F_t(\Delta_t) - F_{t+1}(\Delta_{t+1}) + v_{t}\Delta_t,
    \label{eq:boundRTuFirstStep}
\end{align}
where we used $F_{T+1}(\Delta_{T+1}) - F_{T+1}(u) \leq 0$ in the inequality.
From here, we will bound $F_t(\Delta_t) - F_{t+1}(\Delta_{t+1}) + v_{t}\Delta_t$ by two different ways and then take their minimum.
First, we have 
\begin{nalign}
    &F_t(\Delta_t) - F_{t+1}(\Delta_{t+1}) + v_{t}\Delta_t \\
    &=  F_t(\Delta_t) - (F_t(\Delta_{t+1}) - \frac{1}{2\eta_t}\Delta_{t+1}^2 + v_{t}\Delta_{t+1} + \frac{1}{2\eta_{t+1}}\Delta_{t+1}^2) + v_{t}\Delta_t \\
    &= F_t(\Delta_t) - F_{t}(\Delta_{t+1}) + (\frac{1}{\eta_t} - \frac{1}{\eta_{t+1}})\Delta_{t+1}^2 + v_{t}(\Delta_{t} - \Delta_{t+1}) \\
    &\leq v_{t}(\Delta_{t} - \Delta_{t+1}) \leq 2D_T\abs{v_{t}},
    \label{eq:boundEachTimeStepFirstWay}
\end{nalign}
where the first inequality is due to the following facts:
\begin{itemize}
    \item For $p \leq 1$ and $\alpha_{t+1} \leq \alpha_t$, we have $\eta_{t+1} = \alpha_{t+1} \frac{p^{t}}{\sqrt{\sum_{s=0}^{t} \beta_2^{-s}g_s^2}} \leq \alpha_t \frac{p^{t-1}}{\sqrt{\sum_{s=0}^{t-1} \beta_2^{-s}g_s^2}} = \eta_t$. In other words, the sequence of learning rates $(\eta_t)_t$ is non-increasing.
    \item $F_t(\Delta_t) \leq F_t(\Delta_{t+1})$ by the definition of $\Delta_t$.
\end{itemize}
Next, using either a local-norm analysis as in~\citet[][]{OrabonaIntroToOnlineLearningBook} or the strong convexity of $F_t$, we obtain 
\begin{align}
    F_t(\Delta_t) - F_{t}(\Delta_{t+1}) + v_{t}(\Delta_{t} - \Delta_{t+1})
    &\leq \frac{\eta_t v_{t}^2}{2}.
    \label{eq:boundEachTimeStepSecondWay}
\end{align}
Combining~\Cref{eq:boundEachTimeStepFirstWay} and~\Cref{eq:boundEachTimeStepSecondWay} leads to
\begin{align*}
  F_t(\Delta_t) - F_{t+1}(\Delta_{t+1}) + v_{t}\Delta_t \leq \min\left\{\frac{\eta_t v_{t}^2}{2},  2D_T\abs{v_{t}}\right\}.
\end{align*}
Plugging this into~\Cref{eq:boundRTuFirstStep} and expanding the definition of $\eta_t$, we obtain
\begin{nalign}
     R_T(u) &\leq \frac{u^2}{\eta_{T+1}} + \sum_{t=1}^{T} \min\left\{\frac{\eta_t v_{t}^2}{2},  2D_T\abs{v_{t}}\right\} = \frac{u^2 \sqrt{\sum_{t=1}^T \beta_2^{-t}g_t^2}}{\alpha_{T+1} p^T} + \sum_{t=1}^{T} \min\left\{\frac{\eta_t v_{t}^2}{2},  2D_T\abs{v_{t}}\right\} \\
     &= \frac{u^2 \sqrt{\sum_{t=0}^T \beta_2^{-t}g_t^2}}{\alpha_{T+1} p^T} + \sum_{t=1}^T \min\left\{\frac{\alpha_t}{2} \frac{p^{t-1} v_t^2}{\sqrt{\sum_{s=0}^{t-1} (p^s v_s)^2}}, 2D_T \abs{v_t}\right\}.
    \label{eq:boundRTuSecondStep}
\end{nalign}
We consider two cases.
First, if $p^{t-1} \abs{v_{t}} \geq \sqrt{2}\sqrt{\sum_{s=0}^{t-1} (p^s v_s)^2}$, then 
\begin{align}
    2D_T \abs{v_t} &= 2D_T \frac{\sqrt{2}\abs{v_t} - \abs{v_t}}{\sqrt{2}-1} \leq 2D_T \frac{\sqrt{2}(\abs{v_t}- \frac{\sqrt{\sum_{s=1}^{t-1} (p^s v_s)^2}}{p^{t-1}})}{\sqrt{2}-1} \\
    &\leq 7D_T(\abs{v_t} - \sqrt{\sum_{s=0}^{t-1} v_s^2}) \leq 7D_T(\max_{s \leq t}\abs{v_s} - \max_{s \leq t-1}\abs{v_s}),
    \label{eq:boundMinSecondTerm}
\end{align}
where the the second inequality is due to $\sum_{s=1}^{t-1} (p^s v_s)^2 \geq (p^{t-1})^2 \sum_{s=1}^{t-1}v_s^2$ for $p \leq 1$.

On the other hand, if $p^{t-1} \abs{v_t} < \sqrt{2}\sqrt{\sum_{s=0}^{t-1} (p^s v_s)^2}$, then we have
\begin{align}
    \frac{p^{t-1} v_t^2}{\sqrt{\sum_{s=0}^{t-1}(p^s v_s)^2}} &= \frac{\sqrt{3}}{p^{t+1}} \frac{(p^tv_t)^2}{\sqrt{3\sum_{s=0}^{t-1}(p^s v_s)^2}} \leq \frac{\sqrt{3}}{p^{t+1}} \frac{(p^tv_t)^2}{\sqrt{(p^tv_t)^2 + \sum_{s=1}^{t-1}(p^s v_s)^2}} \\
    &\leq \frac{\sqrt{6}}{p^{t+1}} \left( \sqrt{\sum_{s=0}^{t}(p^s v_s)^2} - \sqrt{\sum_{s=0}^{t-1}(p^s v_s)^2} \right),   
\end{align}
where we used $p^t\abs{v_t} \leq p^{t-1}\abs{v_t}$ in the first inequality and $\frac{x^2}{\sqrt{x^2 + y}} \leq \sqrt{2}(\sqrt{x^2 + y} - \sqrt{y})$ for all $x \in \R, y \geq 0$ in the second inequality.
Consequently,
\begin{align}
  \alpha_t \frac{p^{t-1} v_t^2}{\sqrt{\sum_{s=0}^{t-1}(p^s v_s)^2}} &\leq \frac{\alpha_t\sqrt{6}}{p^{t+1}} \left( \sqrt{\sum_{s=0}^{t}(p^s v_s)^2} - \sqrt{\sum_{s=0}^{t-1}(p^s v_s)^2} \right)  \\
  &\leq \frac{\sqrt{6}}{p} \left(\max_{t \in [T]}\frac{\alpha_t}{p^t}\right) \left( \sqrt{\sum_{s=0}^{t}(p^s v_s)^2} - \sqrt{\sum_{s=0}^{t-1}(p^s v_s)^2} \right)
  \label{eq:boundMinFirstTerm}
\end{align}

Plugging~\Cref{eq:boundMinSecondTerm,eq:boundMinFirstTerm} into~\Cref{eq:boundRTuSecondStep} and summing over $T$ rounds, we obtain
\begin{align}
    R_T(u) \leq \left(\frac{u^2}{\alpha_{T+1} p^T} + \frac{\sqrt{6}}{2p} \left(\max_{t \in [T]}\frac{\alpha_t}{p^t}\right) \right)\sqrt{\sum_{t=0}^T (p^tv_t)^2} + 7D_T \max_{0 \leq t \leq T} \abs{v_t}.
\end{align}
Using $p = \frac{\beta_1}{\sqrt{\beta_2}}$ and $v_t = \beta_1^{-t}g_t$ leads to the desired statement in~\Cref{eq:generalbound}.
\end{proofof}

\subsection{The generalized regret bound is tight}
\label{sec:howtight}
Consider the bounded domain $\gX = [-D, D]$.
We rewrite the order of the generalized regret bound in Corollary~\ref{corolary:constantAlpha} as a function of $p$ and $(v_t)_t$:
\begin{align}
  R_T(u) \lesssim B_{T, \alpha, p, (v_t)_t}(u) \defeq \left(\frac{u^2}{\alpha} + \frac{\alpha}{p}\right) \frac{1}{p^T} \sqrt{\sum_{t=0}^T (p^t g_t)^2} + D\max_{0 \leq t \leq T}\abs{v_t}.
  \label{eq:Bformula}
\end{align}
Recall that this bound holds for \emph{any} choices of $0 < p \leq 1, u \in \gX, \alpha > 0$ and $(v_t)_t$. 
The following theorem states that this bound is tight and cannot be significantly improved further. 

\begin{theorem}
  Let $D > 0$ and $0.4 \leq p \leq 0.6$ be arbitrary universal constants. 
  On the bounded domain $\gX = [-D, D]$, there exists a sequence $(v_t)_t$ where $v_t > 0$ for all $t \in [T]$, such that with $\alpha_t = \alpha = \frac{D}{4}$ and $u = -D$, the sequence of $(\Delta_t)_t$ produced by Algorithm~\ref{algo:AdamfromFTRLWrittenInV} satisfies $R_T(u) = \Omega(B_{T, \alpha, p, (v_t)_t}(u))$.
  \label{theorem:tight}
\end{theorem}
\begin{proof}
  Let $\kappa > 0$ be a constant. 
We select a constant $v_0 > 0$ arbitrarily and define a sequence $v_t = \kappa^t v_0$ for $t = 1, 2, \dots, T$.
Under this sequence $(v_t)_t$, the learning rates are
\begin{align}
  \eta_t &= \alpha\frac{p^{t-1}}{\sqrt{\sum_{s=0}^{t-1} (p^s v_s)^2}} = \alpha\frac{p^{t-1}}{\sqrt{\sum_{s=0}^{t-1} (p^s \kappa^s v_0)^2}} =  \alpha \frac{p^{t-1}}{v_0\sqrt{\sum_{s=0}^{t-1}(p^2\kappa^2)^s}}
  = \alpha \frac{p^{t-1}}{ v_0} \sqrt{\frac{p^2\kappa^2-1}{(p^2\kappa^2)^t-1}}.
\end{align}
The pre-clipping update $\bar{\Delta}_t$ in each round is
\begin{align}
  \bar{\Delta}_t &= -\eta_t \sum_{s=0}^{t-1} v_s = -\eta_t \sum_{s=0}^{t-1}  \kappa^s v_0 = -v_0 \eta_t \frac{\kappa^t-1}{\kappa-1} = -\alpha p^{t-1} \frac{\kappa^t-1}{\kappa-1}\sqrt{\frac{p^2\kappa^2-1}{(p^2\kappa^2)^t-1}}.
  \label{eq:barDeltatformula}
\end{align}
\looseness=-1 Note that this formula of $\bar{\Delta}_t$ holds for any $(\alpha_t)_t = \alpha$ and $\kappa > 0$.

Next, we further require $\kappa \geq \frac{1}{p^2}$. 
The following lemma shows that $\abs{\bar{\Delta}_t} < D$, and thus the clipping operation is never used and can  be safely ignored during our analysis.
\begin{lemma}
  For any $p \in (0, 0.6]$ and $\kappa \geq \frac{1}{p^2}$, we have $\abs{\bar{\Delta}_t} \leq \frac{D}{2}$ holds for all $t \in [T]$.
  \label{lemma:noclippingInTightnessProof}
\end{lemma}
\begin{proof}
  We have $\kappa > 1$ in this case. 
  From~\Cref{eq:barDeltatformula}, we re-write the pre-clipping update as 
  \begin{align*}
    \bar{\Delta}_t = -\alpha \frac{\sqrt{p^2\kappa^2-1}}{p (\kappa-1)} \frac{p^t (\kappa^t - 1)}{\sqrt{(p^2\kappa^2)^t-1}}.
  \end{align*}
  Our proof makes uses of two technical lemmas, Lemma~\ref{lemma:technicalLemma1} and Lemma~\ref{lemma:technicalLemma2}. 
  By Lemma~\ref{lemma:technicalLemma1}, we have $\frac{p^t (\kappa^t - 1)}{\sqrt{(p^2\kappa^2)^t-1}} \leq 1$ for all $\kappa \geq \frac{1}{p^2}$. 
  By Lemma~\ref{lemma:technicalLemma2}, we have $\frac{\sqrt{p^2\kappa^2-1}}{p (\kappa-1)} \leq 2$.
  Therefore,
  $\abs{\bar{\Delta}_t} =  \alpha \frac{\sqrt{p^2\kappa^2-1}}{p (\kappa-1)} \frac{p^t (\kappa^t - 1)}{\sqrt{(p^2\kappa^2)^t-1}} \leq 2\alpha = \frac{D}{2}$.
\end{proof}
Lemma~\ref{lemma:noclippingInTightnessProof} implies that for every round $t$, we have $\Delta_t = \bar{\Delta}_t$.
Moreover, $\Delta_t \geq -\frac{D}{2}$. 
We can then compute a lower bound for the regret $R_T(-D)$.
For $u = -D$, the regret is lower bounded by
\begin{align}
  R_T(-D) = \sum_{t=1}^T v_t(\Delta_t + D) \geq \sum_{t=1}^T v_t (D - \frac{D}{2}) = \frac{v_0 D}{2} \sum_{t=1}^T \kappa^t = v_0 D \kappa \frac{ \kappa^T - 1}{\kappa-1},
  \label{eq:RTminusD}
\end{align}
where the first inequality used $v_t > 0$ and $\Delta_t \geq \frac{-D}{2}$.

Finally, we compute (the order of) the upper bound $B_{T, \alpha, p, (v_t)_t}(-D)$.
Note that $\max_{0 \leq t T}\abs{v_t} = \max_{0 \leq t T}v_0{\kappa^t} = v_0 \kappa^T$. 
With $p \geq 0.4$,~\Cref{eq:Bformula} implies that the regret bound is of order 
\begin{nalign}
 B_{T, \alpha, p, (v_t)_t}(-D) &\simeq (\frac{u^2}{\alpha} + \frac{\alpha}{p}) \frac{1}{p^T} \sqrt{\sum_{t=0}^T (p^t v_t)^2} + D\max_{t \in [0,T]}\abs{v_t} \\
 &\simeq Dv_0 \left(\kappa^T + \frac{1}{p^T}\sqrt{\sum_{t=0}^T (p^2\kappa^2)^{t}}\right) = Dv_0\left(\kappa^T + \frac{1}{p^T}\sqrt{\frac{(p^2\kappa^2)^{T} - 1}{p^2\kappa^2 - 1}}\right).
\end{nalign}
By~\Cref{eq:RTminusD}, $R_T(-D) = \Omega(v_0 D (\kappa^T - 1))$.
From the fact that $\kappa^T - 1 = \Omega(\kappa^T)$ and $\kappa^T = \Omega(\frac{1}{p^T}\sqrt{(p^2\kappa^2)^T-1})$  for $T \geq 2$ and $\kappa > 1$,
we conclude that $R_T(-D) = \Omega(B_{T, \alpha, p, (v_t)_t}(-D))$.
\end{proof}

\begin{remark}
\looseness=-1  An important point of clarification is that, our upper and lower bounds do \emph{not} suggest $\beta_1 = \sqrt{\beta_2}$ is uniformly better than on every problem instance. 
  Under an oblivious adversary, the optimality of setting $\beta_1 = \sqrt{\beta_2}$ comes from minimizing \emph{the regret bound} in Remark~\ref{remark3}, not from minimizing \emph{the regret itself}. In other words, if the problem admits a loss sequence specified in our Theorem~\ref{theorem:tight}, then choosing $\beta_1 = \sqrt{\beta_2}$ is always going to be better than $\beta_1 < \sqrt{\beta_2}$. This, however, does not exclude other loss sequences (still from an oblivious adversary) where setting $\beta_1 < \sqrt{\beta_2}$ is actually better, because the regret upper bound may be loose in that case, and the actual regret might be smaller.
  \label{remark7}
\end{remark}
\begin{remark}
  The worst-case lower bound in our Theorem~\ref{theorem:tight} is algorithm-dependent. This leaves the question of whether there are other algorithms that may achieve better worst-case or problem-dependent regret bounds. For example, an algorithm whose regret depends on the total variation of the sequence of losses may significantly outperform Adam when the loss sequences (i.e. the gradients) are not changing much (i.e. in lazy-training regime where the parameters of a model do not deviate much from the initialization). We leave this as a future work.
  \label{remark8}
\end{remark}

\section{An Analysis for Adam with $\beta_1 \geq \sqrt{\beta_2}$ }
\label{sec:beta1largerthansqrtbeta2}
\looseness=-1 In this section, we extend the technique in the proof of Theorem~\ref{theorem:GeneralRegretBoundPlessthan1} to derive a $\beta_1$-discounted regret bound for Adam with $\beta_1 \geq \sqrt{\beta_2}$.
With a different choice of choice of $(\alpha_t)_t$, we will show a bound that, surprisingly, reduces to the bound in~\citet[][Theorem B.1]{AhnAdamisFTRL2024} when $\beta_1 = \sqrt{\beta_2}$. 
Furthermore, under an oblivious adversary, this bound also suggests that choosing $\beta_2 = \beta_1^2$ is optimal.

With $\beta_1 \geq \sqrt{\beta_2}$, we have $p \geq 1$.
To ensure that the sequence of learning rates $(\eta_t)_t$ is non-increasing, we use $\alpha_t = \frac{\alpha}{p^{t-1}}$, where $\alpha > 0$ is a constant. 
In other words, $\alpha_t$ decay exponentially.
The learning rate $\eta_t$ in~\Cref{algo:AdamfromFTRL,algo:AdamfromFTRLWrittenInV} becomes
\begin{align}
  \eta_t = \alpha_t \frac{p^{t-1}}{\sqrt{\sum_{s=1}^{t-1} (p^s v_s)^2}} = \frac{\alpha}{\sqrt{\sum_{s=1}^{t-1} (p^s v_s)^2}},
  \label{eq:etatWhenplargerthan1}
\end{align}
which is non-increasing.
The pre-clipping update in round $t$ is $\bar{\Delta}_t = -\eta \sum_{s=1}^{t-1} v_s = -\frac{\alpha \sum_{s=1}^{t-1}v_s}{\sum_{s=1}^{t-1}(p^s v_s)^2}$, which could be understood as RMSProp with momentum.
The following theorem states the regret bound of this algorithm.
\begin{theorem}
  For any $\alpha > 0, T \geq 2$, Algorithm~\ref{algo:AdamfromFTRL} with $\beta_1 \geq \sqrt{\beta_2}, \alpha_t = \frac{\alpha}{p^{t-1}}$ guarantees
  \begin{align}
    R_T(u) &\leq \left(\frac{u^2}{\alpha} + \frac{\alpha\sqrt{6}}{2}\right)\sqrt{\sum_{t=1}^T \beta_2^{-t}g_t^2} + 7D_T \max_{0 \leq t \leq T} \abs{\beta_1^{-t} g_t}
    \label{eq:RegretBoundPlargerthan1}
  \end{align}
  \label{theorem:RegretBoundPlargerthan1}
\end{theorem}
Before showing the proof of Theorem~\ref{theorem:RegretBoundPlargerthan1}, we discuss its implication.
\begin{remark}
      Multiplying $\beta_1^T$ on both sides of~\Cref{eq:RegretBoundPlargerthan1} leads to
    \begin{align}
        R_{T, \beta_1}(u) = \beta_1^T R_{T}(u) \leq (\frac{u^2}{\alpha} + \frac{\alpha \sqrt{6}}{2}) \sqrt{\sum_{t=0}^T (\beta_1^2)^T\beta_2^{-t}g_t^2} + 7D_T \max_{t \in [0,T]} \abs{\beta_1^{T-t}g_t}.
        \label{eq:discountedRegretBoundPlargerthan1}
    \end{align}
    Under an oblivious adversary, the $\beta_1$-discounted regret bound in~\Cref{eq:discountedRegretBoundPlargerthan1} is decreasing in $\beta_2$.
    Therefore, setting $\beta_2 = \beta_1^2$ is optimal for an oblivious adversary.
    Furthermore, setting $\beta_2 = \beta_1^2$ recovers both ~\citet[][Theorem B.1]{AhnAdamisFTRL2024} and~\citet[][Theorem 9]{AhnNeurIPS2024AdamisEffectiveNonConvex}.
\end{remark}
\begin{proofof}{Theorem~\ref{theorem:RegretBoundPlargerthan1}}
Since the learning rates $(\eta_t)_t$ in~\Cref{eq:etatWhenplargerthan1} is non-increasing, we can proceed as in the proof of Theorem~\ref{theorem:GeneralRegretBoundPlessthan1} and obtain
\begin{align}
  R_T(u) &\leq \frac{u^2}{\eta_{T+1}} + \sum_{t=1}^{T} \min\left\{\frac{\eta_t v_{t}^2}{2},  2D_T\abs{v_{t}}\right\} \\
  &= \frac{u^2}{\alpha}\sqrt{\sum_{t=0}^T (p^t v_t)^2} + \sum_{t=1}^{T} \min\left\{\frac{\alpha v_{t}^2}{2 \sqrt{\sum_{s=0}^{t-1}(p^sv_s)^2}},  2D_T\abs{v_{t}}\right\}.
  \label{eq:boundRTuSecondStepPlargerthan1}
\end{align}
We consider two cases.
First, if $\abs{v_{t}} \geq \sqrt{2}\sqrt{\sum_{s=0}^{t-1} (p^s v_s)^2}$, then 
\begin{align}
    2D_T \abs{v_t} &= 2D_T \frac{\sqrt{2}\abs{v_t} - \abs{v_t}}{\sqrt{2}-1} \leq 2D_T \frac{\sqrt{2}}{\sqrt{2}-1}\left(\abs{v_t}- \sqrt{\sum_{s=0}^{t-1} (p^s v_s)^2}\right) \\
    &\leq 7D_T\left(\abs{v_t} - \sqrt{\sum_{s=0}^{t-1} v_s^2}\right) \leq 7D_T(\max_{0 \leq s \leq t}\abs{v_s} - \max_{0 \leq s \leq t-1}\abs{v_s}),
    \label{eq:boundMinSecondTermPlargerthan1}
\end{align}
where the the second inequality is due to $\sum_{s=1}^{t-1} (p^s v_s)^2 \geq \sum_{s=1}^{t-1}v_s^2$ for $p \geq 1$.

On the other hand, if $\abs{v_t} < \sqrt{2}\sqrt{\sum_{s=0}^{t-1} (p^s v_s)^2}$, then we have
\begin{align}
    \frac{v_t^2}{\sqrt{\sum_{s=0}^{t-1}(p^s v_s)^2}} &= {\sqrt{3}} \frac{v_t^2}{\sqrt{3\sum_{s=0}^{t-1}(p^s v_s)^2}} \leq {\sqrt{3}} \frac{v_t^2}{\sqrt{v_t^2 + \sum_{s=1}^{t-1}(p^s v_s)^2}} \\
    &\leq {\sqrt{6}} \left( \sqrt{v_t^2 + \sum_{s=0}^{t-1}(p^s v_s)^2} - \sqrt{\sum_{s=0}^{t-1}(p^s v_s)^2} \right) \\
    &\leq {\sqrt{6}} \left( \sqrt{\sum_{s=0}^{t}(p^s v_s)^2} - \sqrt{\sum_{s=0}^{t-1}(p^s v_s)^2} \right),
    \label{eq:boundMinFirstTermPlargerthan1}
\end{align}
where we used $p^{2t} v_t^2 \geq v_t^2$ for $p \geq 1$ in the last inequality.
Plugging~\Cref{eq:boundMinSecondTermPlargerthan1,eq:boundMinFirstTermPlargerthan1} into~\Cref{eq:boundRTuSecondStepPlargerthan1} and summing over $T$ rounds, we obtain, we obtain 
\begin{align}
    R_T(u) \leq \left(\frac{u^2}{\alpha} + \frac{\alpha\sqrt{6}}{2} \right) \sqrt{\sum_{t=0}^T (p^t v_t)^2} + 7D_T\max_{t \in [0,T]}\abs{v_t}..
\end{align}
Using $p = \frac{\beta_1}{\sqrt{\beta_2}}$ and $v_t = \beta_1^{-t}g_t$ leads to the desired statement in~\Cref{theorem:RegretBoundPlargerthan1}.
\end{proofof}

\section{When is setting $\beta_1 = \sqrt{\beta_2}$ not optimal? An Example with Non-Oblivious Adversary}
\label{sec:notoptimalNonObliviousAdversary}
When using Adam for optimizing a target function in the online-to-nonconvex objective $F(\cdot)$, for a fixed $\beta_1$, different values of $\beta_2$ inevitably leads to different sequence of updates $(\Delta_t)_t$, which in turn leads to different sequence of $w_t$.
Consequently, the sequence of gradients $(g_t)_t$ also varies with different values of $\beta_2$. 
It follows that the adversary is non-oblivious, since $g_{t+1}$ depends on the outputs $(\Delta_s)_{s=1,2,\dots,t}$ of the algorithm in past $t$ rounds.

\textbf{Setup.} While our results in the previous sections indicate that setting $\beta_1 = \sqrt{\beta_2}$ is optimal under an oblivious adversary, they do not have any implication for an non-oblivious adversary.
In this section, we present a result showing that $\beta_1 = \sqrt{\beta_2}$ may no longer be optimal when the adversary is non-oblivious.
To this end, we consider the bounded domain $\gX = [-1, 1]$, i.e. $D=1$.
Let $a, b \in (0,1)$ be two universal constants, $a \neq b$.
Let $K = \max\{\frac{1}{1-a}, \frac{1}{1-b}\}$, and $\alpha = \frac{1}{K}$. 

Fix $\beta_1 \in (0,1)$.
We will compare the regret of two instances of~\Cref{algo:AdamfromFTRLWrittenInV}, denoted by $\gA$ and $\gA'$. 
Algorithm $\gA$ uses $\beta_1 = p \sqrt{\beta_2}$, where $p < 1$. 
Algorithm $\gA'$ uses $\beta_1 = \sqrt{\beta_2}$.
Both algorithms use $\alpha_t = \alpha = \frac{1}{K}$.

\textbf{Non-oblivious Adversary.}
Fix an arbitrary $v > 0$ and define two sequences $v_t = a^t v$ and $v'_t = b^t v$ for $t = 0, 1, 2, \dots$.
Note that $v_0 = v'_0 = v$, which resembles practical scenarios where the very first gradients evaluated when the model has just been initialized.
For algorithms $\gA$ and $\gA'$, the adversary will use the losses $(v_t)_t$ and $(v'_t)_t$, respectively.

\textbf{Regret Analysis.}
Let $\Delta_t$ and $\Delta'_t$ be the updates in round $t$ of $\gA$ and $\gA'$, respectively. 
Because all of losses $(v_t)_t$ and $(v'_t)_t$ are positive, the optimal updates for both $\gA$ and $\gA'$ is $u = -D$.
The following theorem shows that the regret of $\gA$ is strictly smaller than that of $\gA'$, which indicates that setting $\beta_1 = \sqrt{\beta_2}$ is not optimal.

\begin{theorem} Let $u = -D$. For any $T \geq 2$, the sequence of updates from $\gA$ and $\gA'$ satisfy
  \begin{align}
    \sum_{t=1}^T v_t(\Delta_t - u) < \sum_{t=1}^T v'_t (\Delta'_t - u).
  \end{align}
  \label{theorem:notoptimal}
\end{theorem}
\begin{proof}
First, we show that no clipping is required in both $\gA$ and $\gA'$. 
Since both algorithms use the same constant $\alpha = \frac{1}{K}$ for the sequence $(\alpha_t)_t$, we can apply~\Cref{eq:barDeltatformula} twice, once for $\kappa = a$ and once for $\kappa = b$, to obtain 
\begin{align}
  \bar{\Delta}_t &= -\frac{p^{t-1} \sqrt{1-(pa)^2}}{K\sqrt{1- (pa)^{2t}}} \frac{1-a^t}{1-a}, \\
  \bar{\Delta'}_t &= -\frac{p^{t-1} \sqrt{1-(pb)^2}}{K\sqrt{1- (pb)^{2t}}} \frac{1-b^t}{1-b},
\end{align}
where $\bar{\Delta}_t$ and $\bar{\Delta'}_t$ are the pre-clipping updates in round $t$ of $\gA$ and $\gA'$, respectively.
Since $D = 1$, it suffices to show that $\abs{\bar{\Delta}_t} \leq 1$ and $\abs{\bar{\Delta'}_t} \leq 1$.
We have $\abs{\bar{\Delta}_t} \leq 1$ follow from $K(1-a) \geq 1$ and
\begin{align}
  p^{t-1}(1-a^t)\sqrt{1-(pa)^2} \leq \sqrt{1-(pa)^2} \leq \sqrt{1-(pa)^{2t}}
\end{align}
for any $t \geq 1$ and $pa \leq 1$. 
The same argument applies for $\bar{\Delta'}_t$. 
We conclude that no clipping happens, which implies that  $\Delta_t = \bar{\Delta}_t$ and $\Delta'_t = \bar{\Delta'}_t$ 

Let $u = -D$ be the comparator for both $\gA$ and $\gA'$.
The regret in each round $t$ of the algorithm $\gA$ and $\gA'$ are, respectively,
\begin{align}
  f_{\gA}(t,u) &\defeq v_t(\Delta_t - u) =  va^t\left( 1 - \frac{p^{t-1} \sqrt{1-(pa)^2}}{K\sqrt{1- (pa)^{2t}}} \frac{1-a^t}{1-a}  \right), \\
  f_{\gA'}(t,u) &\defeq v'_t(\Delta'_t - u) =  vb^t\left( 1 - \frac{\sqrt{1-b^2}}{K\sqrt{1- b^{2t}}} \frac{1-b^t}{1-b}  \right).
\end{align}
Finally, we specify a sufficient condition for $a$ and $b$ such that the per-round regret of $\gA$ is smaller than that of $\gA'$.
\begin{lemma}
  For any $a, b \in (0,1)$ such that $a < b^2$, we have $f_{\gA}(t,u) < f_{\gA'}(t,u)$ for all $t \geq 1$.
  \label{lemma:ftgt}
\end{lemma}
\begin{proof}
  Since $a < b^2 < b$, we have $1-a > 1-b$, thus $K = \max\{\frac{1}{1-a}, \frac{1}{1-b}\} = \frac{1}{1-b}$. 
  Hence,
  \begin{align}
    f_{\gA}(t,u) &= a^t\left( 1 - \frac{p^t \sqrt{1-(pa)^2}}{K\sqrt{1- (pa)^{2t}}} \frac{1-a^t}{1-a}  \right) \leq a^t < (b^2)^t = b^t (1 - (1-b^t)) \\
    &= b^t \left( 1 - \frac{1-b^t}{K(1-b)} \right) \leq b^t\left( 1 - \frac{\sqrt{1-b^2}}{K\sqrt{1- b^{2t}}} \frac{1-b^t}{1-b}  \right) = f_{\gA'}(t,u),
  \end{align}
  where the last inequality is due to $1-b^2 \leq 1-b^{2t}$ for any $b \in (0,1)$.
\end{proof}
The desired statement in~\Cref{theorem:notoptimal} immediately follows from Lemma~\ref{lemma:ftgt} and summing up over $t \in [T]$.
\end{proof}


\section{Conclusion and Future Works}
In this work, we studied the Adam optimizer from an online learning perspective. 
By considering Adam as an instance of FTRL,
we derived more general discounted regret bounds for Adam that hold beyond the restrictive setting of $\beta_1 = \sqrt{\beta_2}$ often required in existing works.
For both cases $\beta_1 \geq \sqrt{\beta_2}$ and $\beta_1 \leq \sqrt{\beta_2}$, our new analyses and their bounds strictly generalize existing results.
Moreover, we show that our bounds are worst-case tight and cannot be significantly improved furthers.
Our results imply that when using Adam in the online-to-nonconvex framework, a rigorous approach towards tuning the two momentum factors $\beta_1$ and $\beta_2$ would require an exact modelling of the adversary. 
Future works include characterizing the adversary, i.e. how the sequence of gradients changes according to different values of $\beta_1, \beta_2$, on popular convex and non-convex objectives.

\acks{We thank Ali Mortazavi and Nishant Mehta for useful discussion on the proof and interpretation of Theorem 1. We thank the helpful suggestions from all reviewers. In particular, our Remarks~\ref{remark7} and~\ref{remark8} on clarifying the significance of our lower bound come from our rebuttal answer to reviewer 2J4Y.}

\bibliography{momentumAdamOL.bib}

\appendix




\section{Technical Lemmas}
\begin{lemma}
  For any $x \in (0,1], y \geq \frac{1}{x^2}$, we have $ \frac{x^t(y^t - 1)}{\sqrt{(x^2y^2)^t - 1}} \leq 1$ for all $t \geq 1$.
  \label{lemma:technicalLemma1}
\end{lemma}
\begin{proof}
  The inequality is equivalent to 
  \begin{align*}
    & x^{2t}(y^{2t} - 2y^t + 1) \leq x^{2t}y^{2t} - 1 \\
    \Leftrightarrow \quad & x^{2t}(2y^t - 1) \geq 1 \\
    \Leftrightarrow \quad &  2y^t \geq 1 + \frac{1}{x^{2t}}.
  \end{align*}  
  The last inequality is true due to $2y^t \geq \frac{2}{x^{2t}}  \geq 1 + \frac{1}{x^{2t}}$ holds for all $x \in (0,1]$.
\end{proof}

\begin{lemma}
  For any $x \in (0,0.6], y \geq \frac{1}{x^2}$, we have $\frac{\sqrt{x^2y^2-1}}{x(y-1)} \leq 2$.
  \label{lemma:technicalLemma2}
\end{lemma}
\begin{proof}
  The inequality is equivalent to 
  \begin{align*}
    & 4x^{2}(y^{2} - 2y + 1) \geq x^{2}y^{2} - 1 \\
    \Leftrightarrow \quad & x^{2}(3y^2 - 8y + 4) \geq -1.
  \end{align*}  
  The last inequality is true due to $y \geq \frac{1}{x^{2}}  > \frac{8}{3}$ holds for all $x \in (0,0.6]$.
\end{proof}


\section{A Recent Empirical Finding From~\citet{Orvieto2025SearchAdamSecretSauce}}
\begin{figure}[h]
  \centering
  \includegraphics[scale=0.75]{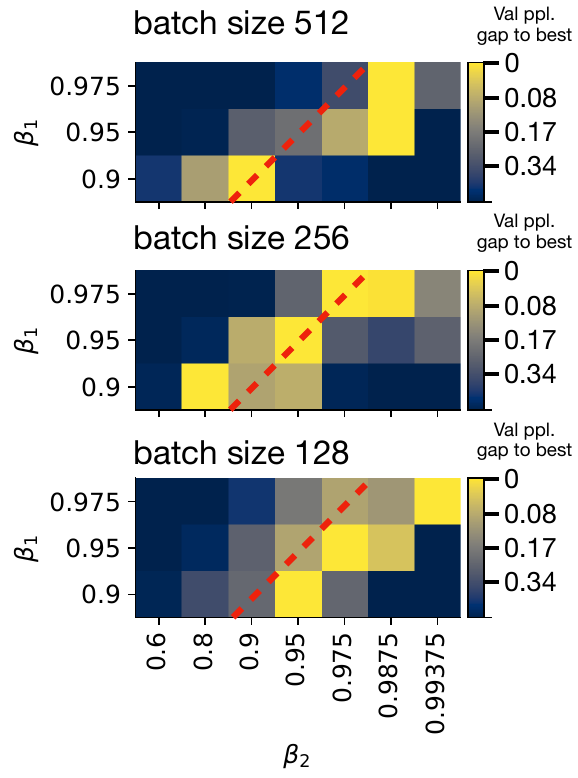}
  \caption{This is Figure 3 in~\citet{Orvieto2025SearchAdamSecretSauce}, demonstrating the empirical results of tuning $\beta_1, \beta_2$ across three batch sizes for training $160M$-parameter transformers.
  Yellow indicates optimal performances, while dark blue indicates sub-optimal performances. The smallest $\frac{\beta_1}{\sqrt{\beta_2}}$ ratio of a yellow box is approximately $1$, achieved at batch size $256$, $\beta_1 = 0.9$ and $\beta_2 = 0.8$.}
  \label{fig:betascorr}
\end{figure}

\end{document}